%% file: template_isit24.tex
\documentclass[conference,letterpaper]{IEEEtran}

\usepackage[utf8]{inputenc} 
\usepackage[T1]{fontenc}
\usepackage{url}
\usepackage{ifthen}
\usepackage{cite}
\usepackage[cmex10]{amsmath} 
                             
\usepackage{xcolor}
\usepackage{color,soul}
\usepackage{amssymb}
\usepackage{amsthm}
\usepackage{cite}

\input{commands}

\IEEEoverridecommandlockouts
\begin{document}
\title{Supervised Contrastive Representation Learning:\\  Landscape Analysis with Unconstrained Features} 

\author{%
 \IEEEauthorblockN{Tina Behnia, Christos Thrampoulidis}
 \IEEEauthorblockA{Department of Electrical and Computer Engineering\\
                   University of British Columbia\\
                   Vancouver, BC, Canada\\
                   Email: \{tina.behnia, cthrampo\}@ece.ubc.ca}\thanks{This work is supported by an NSERC Discovery Grant, a CRG8-KAUST award, and funding from AML-TN program at UBC.}
}



\maketitle

\begin{abstract}
    Recent findings reveal that over-parameterized deep neural networks, trained beyond zero training-error, exhibit a distinctive structural pattern at the final layer, termed as Neural-collapse (NC). These results indicate that the final hidden-layer outputs in such networks display minimal within-class variations over the training set. While existing research extensively investigates this phenomenon under cross-entropy loss, there are fewer studies focusing on its contrastive counterpart, supervised contrastive (SC) loss. Through the lens of NC, this paper employs an analytical approach to study the solutions derived from optimizing the SC loss. We adopt the unconstrained features model (UFM) as a representative proxy for unveiling NC-related phenomena in sufficiently over-parameterized deep networks.  
    We show that, despite the non-convexity of SC loss minimization, all local minima are global minima. Furthermore, the minimizer is unique (up to a rotation). We prove our results by formalizing a tight convex relaxation of the UFM. Finally, through this convex formulation, we delve deeper into characterizing the properties of global solutions under label-imbalanced training data. 

\end{abstract}

\section{Introduction}
\input{sections/intro}
\section{Problem Setup}
\input{sections/setup}

\section{Results}
In this section, we perform an analysis of the local minimizers of the SC loss with the UFM \eqref{eq:SCL_UFM}. As part of the analysis, we rely on the characterization of a tight convex relaxation of \eqref{eq:SCL_UFM} which we use later for taking a closer look at the structure of the global solutions. We first start by stating a general property of the local optimal points of \eqref{eq:SCL_UFM}.
{
\subsection{NC Holds at Local Optimal Solutions}
}
\input{sections/SCL_stationary}
{
\subsection{UFM: Landscape and Optimality Conditions}\label{sec:scl_land}
}
\input{sections/SCL_ufm}


\section{Discussion}
\input{sections/discussion}
\bibliographystyle{IEEEtran}
\bibliography{refs}









\newpage
\appendices
\onecolumn
\section{Proof of \lem\ref{lem:stationary_NC}}
\input{sections/lem1_pf}

\section{Proof of \lem\ref{lem:unique_main}}
\input{sections/lem2_pf}

\section{Proof of \lem\ref{lem:scl_symm_global}}
\input{sections/lem3_pf}

\end{document}

%% file: commands.tex
\usepackage{mathtools}

\usepackage{dsfont}

\usepackage{nicefrac} 
\definecolor{darkred}{RGB}{150,0,0}
\definecolor{darkgreen}{RGB}{0,150,0}
\definecolor{darkblue}{RGB}{0,0,200}

\newtheorem{theorem}{Theorem}

\newtheorem{lemma}{Lemma}

\newtheorem{corollary}{Corollary}
\newtheorem{propo}{Proposition}
\newtheorem{definition}{Definition}

\newtheorem{remark}{Remark}

\renewcommand{\paragraph}[1]{\textbf{#1}.~}

\newenvironment{fminipage}%
  {\begin{Sbox}\begin{minipage}}%
  {\end{minipage}\end{Sbox}\fbox{\TheSbox}}

\newcommand{\Id}{\mathds{I}}
\newcommand{\Jd}{\mathds{J}}
\newcommand{\ones}{\mathbf{1}}

\newcommand{\R}{\mathds{R}}
\newcommand{\thm}{Theorem }

\newcommand{\cor}{Corollary }

\newcommand{\lem}{Lemma }

\newcommand{\defn}{Definition }

\newcommand{\rem}{Remark }

\newcommand{\ab}{\mathbf{a}}
\newcommand{\Ab}{\mathbf{A}}
\newcommand{\Bb}{\mathbf{B}}
\newcommand{\db}{\mathbf{d}}

\newcommand{\eb}{\mathbf{e}}
\newcommand{\Gb}{\mathbf{G}}
\newcommand{\gb}{\mathbf{g}}
\newcommand{\Hb}{\mathbf{H}}
\newcommand{\hb}{\mathbf{h}}
\newcommand{\Mb}{\mathbf{M}}
\newcommand{\Pb}{\mathbf{P}}
\newcommand{\Qb}{\mathbf{Q}}
\newcommand{\pb}{\mathbf{p}}
\newcommand{\qb}{\mathbf{q}}
\newcommand{\Rb}{\mathbf{R}}

\newcommand{\ub}{\mathbf{u}}

\newcommand{\Vb}{\mathbf{V}}
\newcommand{\vb}{\mathbf{v}}

\newcommand{\xb}{\mathbf{x}}

\newcommand{\Gammab}{\mathbf{\Gamma}}
\newcommand{\Deltab}{\mathbf{\Delta}}
\newcommand{\mub}{\boldsymbol{\mu}}

\newcommand{\thetab}{{\boldsymbol{\theta}}}

\newcommand{\lambdab}{\boldsymbol{\lambda}}

\newcommand{\Cc}{\mathcal{C}}

\newcommand{\Lc}{\mathcal{L}}

\newcommand{\Sc}{\mathcal{S}}

\newcommand{\Gbtilde}{\widetilde{\Gb}}

\newcommand{\rhobar}{\overline{\rho}}

\newcommand{\nn}{\nonumber}

\providecommand{\norm}[1]{\lVert#1\rVert}
\newcommand{\tr}{\operatorname{tr}}
\newcommand{\diag}{\operatorname{diag}}
\newcommand{\rank}{\operatorname{rank}}

\newcommand{\nmin}{n_{\operatorname{minor}}}
\newcommand{\nmaj}{n_{\operatorname{maj}}}

\newcommand{\Hbopt}{\Hb^*}

\newcommand{\opt}[1]{#1^*}

%% file: sections/intro.tex
%
%

{Contrastive techniques (e.g., \cite{chen2020simple,tian2020contrastive}) have led to major improvements in the self-supervised representation learning, and recently, \cite{khosla2020supervised} introduced the supervised contrastive (SC) loss as an extension of these techniques tailored to the supervised problems. SC loss is applied on the embeddings generated by the network and it penalizes the model by measuring how similar or dissimilar the learned embeddings are for different samples based on their class membership. The objective with the SC loss is to bring the embeddings with the same label close to each other while keeping them apart from those belonging to other classes. 
%
In other words, the loss value emphasizes on the interactions of pairs of samples from the training set.
This is in contrast to other widely used loss functions like cross-entropy (CE) where the loss is computed on each sample individually, aiming to maximize the class-probability associated with the correct label. 
}

{
Focusing on the CE loss, \cite{NC} recently discovered a common pattern at the final layer of the over-parameterized deep-nets, whose number of parameters far exceeds the number of training samples. According to their observations, in the modern deep learning regime where deep-nets are trained beyond zero-training error, at the last layer, the network learns to compress the training data by forcing the embeddings to collapse to their class-means, exhibiting a low-rank structure. Furthermore, when the training data is balanced, the class-means exhibit a symmetric pattern. In particular, they form a simplex \emph{equiangular tight frame} (ETF), where all the class-mean vectors have equal norms and are maximally-separated. Importantly, this {simple} geometry appears to be ``\emph{cross-situationally invariant}'', as it appears at the final layer of various architectures trained on benchmark datasets \cite{NC}.
{Seeking to explain the optimality of this structure, \cite{NC} links the feature learning process to  an abstract optimal coding task. Treating each embedding as a codeword for the class membership, and the classification task as a decoding of the noisy codeword, \cite{NC} demonstrates that this shared geometry yields the optimal codebook. }
This discovery, termed as Neural collapse (NC), has led to a long line of follow-up works, e.g., \cite{fang2021exploring,galanti2021role,graf2021dissecting,han2021neural,hui2022limitations,ULPM,lu2020neural,mixon2020neural,tirer2022extended,xie2022neural,zhu2021geometric,zhou2022optimization}. The majority of these works focus on the CE and mean-squared-error loss, and a few consider analyzing contrastive techniques in this context. 
}


The NC discovery was an empirical attempt to demystify the black-box nature of the deep-nets and to more formally characterize what large networks learn at the training stage. NC also shares connections with the implicit bias literature (e.g., \cite{soudry2018implicit,ji2018risk,lyu2019gradient}) where the success of over-parameterization is attributed to the optimization algorithms, arguing that gradient-based algorithms inherently favor solutions with certain properties that generalize well. However, the emergence of NC in deep overparameterized networks remains hard to theoretically explain due to the complex interactions of their parameters and layers. This challenge motivated a series of theoretical studies to justify the NC phenomenon  from an optimization perspective (e.g., \cite{zhu2021geometric,seli,graf2021dissecting,behnia2023implicit,SCL+,mixon2020neural,fang2021exploring,zhou2022optimization,ULPM,zhou2022all}). These works rely on a proxy model called \emph{unconstrained features model} (UFM). This model leverges the powerful expressivity of large deep-nets to simplify the mathematical formulation of the training objective. This is made possible by only focusing on the interactions of the last-layer parameters of the network in the loss function. Despite its simplicity, UFM has proved a useful tool for predicting certain empirical patterns learned by deep-nets during training.




\subsection{Related Works}

\noindent\textbf{Neural-collapse.}~
After the influential paper \cite{NC}, important follow-up works introduced UFM \cite{mixon2020neural,fang2021exploring,graf2021dissecting} as an abstract model for theoretically studying the last-layer of over-parameterized deep-nets. In this model the focus is only on the structure of the parameters in the final layer that optimize the desired objective. Previous works \cite{mixon2020neural,fang2021exploring,graf2021dissecting,zhu2021geometric,zhou2022optimization,ULPM,seli,zhou2022all} have studied the UFM problem for the CE loss showing that with class-balanced training sets all the global solutions satisfy the NC property and form a simplex ETF as empirically reported by \cite{NC}. Despite being a non-convex problem, \cite{zhu2021geometric} proves that UFM with CE loss has a benign landscape: all local optimals are global optimal, and there are no spurious saddle points. Thus, any optimization algorithm that can escape saddle points, converges to the global optimal, that is the simplex ETF. While most of these works study the balanced setup, several works use the UFM to predict the learned final-layer embeddings for class-imbalanced training data \cite{fang2021exploring,seli,behnia2023implicit}, and support their theoretical findings by experiments on benchmark datasets and architectures. The geometric characterization of the embeddings under the mean-squared \cite{mixon2020neural,zhou2022optimization,tirer2022extended} and contrastive \cite{graf2021dissecting, fang2021exploring} loss have also been investigated by adopting the UFM. Several works \cite{hui2022limitations,galanti2021role,galanti2022improved} also explore the potential links of NC to generalization and down-stream tasks.

\noindent\textbf{Supervised contrastive loss.}~Unlike CE and mean-squared-error loss, there are fewer works on the geometry characterization of the constrastive loss through the lens of NC. 
The SC loss is known to exhibit more robustness to noise and hyper-parameter values \cite{khosla2020supervised}. Its resilience has inspired a series of works where the variants of the SC loss combined with CE are used for training models that tackle distribution shifts in the data with better performance \cite{jitkrittum2022elm,samuel2021distributional,gunel2020supervised,kang2021exploring,li2022targeted}. 
In the context of self-supervised learning, \cite{wang2020understanding} investigates the properties of the contrastive feature learning.
For supervised classification, \cite{graf2021dissecting} is the first to compare how the CE and SC loss encode the training data by studying the UFM. They show that with balanced training data, the global optimizer of the SC loss is ETF similar to CE and demonstrate that deep-nets trained with the SC loss reach this global optimizer. However, unlike CE, they emphasize on the fact that the optimality of this geometry can be violated depending on the mini-batch selection in practical scenarios. {Later, \cite{SCL+} suggests adjustments on the training to ensure learning an embedding geometry that is invariant to the randomness in batching and level of class-imbalance. }
Furthermore, \cite{yaras2022neural} conjecture that as with the CE loss, the UFM trained by the SC loss also has a benign landscape with no spurious local minima despite its non-convex nature. Their  claim is motivated by numerical verifications that show we can reach the global solution of the UFM by running simple gradient methods for minimizing the SC loss.


\subsection{Contributions}
In this work, we theoretically analyze the local solutions of the SC loss under the UFM assumptions. We first demonstrate that all the stationary points exhibit the collapse property (see \defn\ref{def:NC}). Leveraging this low-rank structure at the stationary points, we prove that the SC loss has a benign landscape: we show that although the optimal solution is not unique, all the local solutions only differ up to a rotation and are equivalent to the global solution. 
Our results rely on introducing a tight convex relaxation of the UFM. 
As a consequence of this analysis, we are able to simplify finding the global optimizer into a lower-dimensional equivalent program under the STEP imbalanced assumption on the training set. For the special case of balanced training data, this program reveals ETF as the optimal geometry as also previously shown in \cite{graf2021dissecting}. However, ETF is not a global solution otherwise. 

\noindent\textbf{Notations.}~We use $[n]$ to denote the set $\{1,2,...,n\}$. For matrix $\Vb_
{m\times n}$, $\Vb_{i,j}$ denotes its $(i,j)$-th entry, $\vb_j$ the $j$-th \emph{column}, and $\Vb^\top$ its transpose. When $m=n$, $\Vb\succeq0$ is a PSD matrix and $\diag(\Vb)$ is a diagonal matrix that includes the diagonal elements of $\Vb$.
We use $\ones_m$ and $\Id_m$ to denote the $m$-dimensional vector of all ones and the identity matrix respectively. Finally, $\otimes$ refers to the Kronecker product.

%% file: sections/setup.tex
{Consider a supervised classification setup with $k$ classes and a training set $\Sc = \{(\xb_i,y_i): i \in [n]\}$ 
where $\xb_i\in\R^p$ and  $y_i\in[k]$ are the $i$-th data sample and its class label. 
We also let $n_c\geq1$, $c\in[k]$ to be the number of examples in class $c$.
}

{We consider the deep-net as a non-linear encoder $\hb_\thetab:\R^p\to\R^d$ mapping the input space to a new embedding space that is parameterized by the trainable variables $\thetab$. Define $\hb_{\thetab,i}:=\hb_\thetab(\xb_i),\,i\in[n]$ and let
$\Hb_\thetab:=[\hb_{\thetab,1},...,\hb_{\thetab,n}]_{d\times n}$ be the training embeddings learned by the deep-net. With the SC loss, we seek to find the optimal parameters by minimizing $\Lc_{\text{SC}}$ defined as follows:}
\begin{align}\label{eq:scl_loss}
    \Lc_\text{SC}(\Hb_\thetab;\tau)={\frac{1}{n}}\sum_{i\in[n]}\frac{-1}{n_{y_i}-1}\sum_{\substack{j \neq i\\ y_j = y_i}}\log\bigg( \frac{e^{{\hb_{\thetab,i}^\top\hb_{\thetab,j}} / {\tau}}}{\sum\limits_{\ell \neq i} e^{\hb_{\thetab,i}^\top\hb_{\thetab,\ell} / \tau}}\bigg).
\end{align}
{Here, the scalar $\tau$ is a positive temperature hyper-parameter. This parameter was introduced in the loss by \cite{khosla2020supervised} as they observed it can boost the performance of the model by controlling the smoothness of the loss.}

By minimizing the SC loss, we favor parameters $\thetab$ that pull the examples belonging to the same class closer while simultaneously pushing the samples from different classes away from each other. Training a model with this loss usually involves two stages: first we learn the embeddings $\hb_\thetab(\cdot)$ by minimizing the SC loss, and second, we train a (simple) classifier separately on top of the learned representations. In this work, we focus on the first representation learning stage.

\begin{remark}\label{rem:norm}
    In the definition of the SC loss, introduced by \cite{khosla2020supervised}, the embeddings $\hb_\thetab(\cdot)$ are normalized to the unit sphere. In other words, the loss is a function of the cosine similarity of the embeddings learned by the deep-net. Alternatively, we can assume the deep-net at the final layer is followed by a normalization layer that maps the embeddings to the unit sphere so that  $\norm{\hb_\thetab(\xb_i)}_2 = 1$.
\end{remark}

%

\subsection{Unconstrained Features Model (UFM)}

{In order to develop a theory for the training dynamics of deep-nets, we ought to capture the interplay between several non-linear hidden layers of the network. This makes the theoretical analysis of deep-networks challenging and motivates us to simplify the interactions between layers. UFM \cite{mixon2020neural,fang2021exploring,graf2021dissecting} serves as an abstract model that gives (partial) theoretical justification of the discovery made by \cite{NC}. Focusing on the last layer of the network, UFM relies on the power of over-parameterization, and assumes that large deep-nets are expressive enough to (approximately) map any point in the input space to any point in the embedding space. Hence, in this model, we assume deep-nets can learn \emph{unconstrained features} and treat the last-layer embeddings as free optimization parameters. In other words, we relax the dependence of the embeddings $\Hb_\thetab$ on the network parameters $\thetab$ and optimize the problem directly over $\Hb$. Thus, under UFM, our problem reduces to:}
\begin{align}\label{eq:SCL_UFM_1}
    \opt{\Hb} \in \arg\min_{\Hb} \Lc_{\text{SC}}
    (\Hb;\tau)~~\text{{subj. to}}~\,\|\hb_i\|^2_2 = 1,\, \forall i\in[n].
\end{align}

With this simplification, we study the patterns appearing in large deep-nets, restricting our attention to the interactions of the only element that the NC property is concerned with: the last-layer embeddings. It is also important to highlight that the objective of this optimization problem is still nonconvex as the loss is computed over the non-linear interactions of embeddings, i.e., $\Hb^\top\Hb$.
%
\begin{remark}
    Throughout this paper, we use a slightly different terminology than the Neural-collapse literature. We borrow the terminology from \cite{behnia2023implicit}: we use \emph{Neural-collapse} (NC) particularly when referring to the collapse of the embeddings, and {(implicit) geometry}
    when referring to the arrangement of the last-layer feature-embeddings $\Hb$ (e.g., ETF in balanced case). More formally, by NC, we mean the following:
\end{remark}
\begin{definition}[Neural-collapse (NC) property]\label{def:NC}
    Consider the embeddings matrix $\Hb$. NC holds for $\Hb$ if $\hb_i=\hb_j$, $\forall\, i,j: y_i=y_j$. 
\end{definition}

\subsection{Practices in Contrastive Learning}\label{sec:scl_model}
\input{sections/intro_SCL}

%% file: sections/intro_SCL.tex
%
Contrastive learning involves several heuristics for achieving optimal performance in practice. We choose to overlook some of these common practices to simplify the analysis. We elaborate on these assumptions in this section. While not being fully consistent with the practice, we believe this preliminary analysis can provide useful insights for future works. 

\noindent\textbf{Computational cost.}~Due to the pair-wise comparison of the embeddings in the loss function, SC loss has a high computational cost when computed over the whole training set as in \eqref{eq:scl_loss}. As a result, in practice, the loss is approximated over random batches chosen from the training set, instead of the whole set $\Sc$. However, in our theory we consider the full-batch scenario \eqref{eq:scl_loss}.

\noindent\textbf{Projection head.}~At the training stage, the SC loss is not directly computed on the embeddings. In fact, the embeddings are passed through a projection layer (which can be as simple as a single-layer or shallow MLP), and the loss is minimized on the projected embeddings. This projection layer is discarded at test time, as it is empirically shown that using the pre-projection embeddings yield better performance for downstream tasks. In our analysis, we ignore the projection head of the network, and focus on studying the embeddings on which the loss is directly computed. This is similar to the analysis performed by previous works \cite{graf2021dissecting,fang2021exploring,wang2020understanding}.

\noindent\textbf{Embedding constraints.}~As discussed in \rem\ref{rem:norm}, the SC loss is computed over unit norm embeddings. Since, constraining the norms to the sphere yields a non-convex feasibility set, we relax the norm constraints to a ball of radius 1. On the other hand, we can view the temperature parameter $\tau$ as a scaling on the embeddings $\hb_i$. Specifically, it is easy to see that $\Lc_\text{SC}(\Hb;\tau) = \Lc_\text{SC}(\widetilde\Hb;1)$ where $\widetilde\Hb=\Hb/\sqrt{\tau}$ and as a result $\norm{\widetilde\hb_i}^2_2=1/\tau$ for all $i\in[n]$. In other words, after relaxing the constraints, we study the UFM formulated as follows,
%
%
%
%
%
\begin{align}\label{eq:SCL_UFM}
    \opt{\Hb} \in \arg\min_{\Hb} \Lc_{\text{SC}}
    (\Hb)~~\text{{subj. to}}~\,\|\hb_i\|^2_2 \leq \frac{1}{\tau},\, \forall i\in[n],
\end{align}
where $\Lc_{\text{SC}}
    (\Hb):= \Lc_{\text{SC}}
    (\Hb;\tau=1)$.
Here we recall that, the minimization is over the \emph{unconstrained} embeddings $\Hb$ that do not depend on the network parameters $\thetab$. Although we simplified the problem to deal with a convex feasibility set, the problem \eqref{eq:SCL_UFM} still has a non-convex objective function.

%% file: sections/SCL_stationary.tex

    


{
Let $n_\text{max}=\max_{c\in[k]} n_c$ be the size of the largest class, and recall the necessary condition for local optimality in a general optimization with convex constraints.
\begin{propo}[Necessary stationarity condition {\cite[Chapter 2]{bertsekas1999}}]\label{cond:nec_local}
Consider the optimization problem $\min_{\xb\in\Cc} f(\xb)$ over a convex set $\Cc$. Let $\widehat{\xb}\in\Cc$ be a local minimum. Then, 
    \begin{align}
         & \nabla f(\widehat{\xb})^\top (\xb-\widehat{\xb}) \geq 0, &\quad\forall \xb\in\Cc.\label{eq:1storder}
    \end{align}
\end{propo}
\begin{lemma}\label{lem:stationary_NC}
    Assume $\tau>{2}/{\log({(n-1)}/{(n_\text{\emph{max}}-1)})}$. Let $\widehat{\Hb}$ be locally optimal in the UFM problem \eqref{eq:SCL_UFM}. The NC property (\defn\ref{def:NC}) holds at $\widehat{\Hb}$.
\end{lemma}
\lem\ref{lem:stationary_NC} shows all the feasible points that are locally optimal have a low-rank structure and satisfy the collapse property. Thus, similar to the case of CE \cite{mixon2020neural,zhu2021geometric}, it follows that at the global solutions, the embeddings collapse to their class-means. 

The proof of this lemma simply follows by showing that condition \eqref{eq:1storder} does not hold for a feasible $\widehat\Hb$ that does not satisfy the NC property (recall that the feasible set in \eqref{eq:SCL_UFM} is convex). Specifically, suppose $y_1=y_2$ and $\widehat\hb_1\neq \widehat\hb_2$. Then, define $\Hb$ by swapping the first two columns of $\widehat\Hb$, i.e., $\Hb=[\widehat\hb_2,\widehat\hb_1,\widehat\hb_3,\cdots,\widehat\hb_n]$. It is then possible to show that if the condition on $\tau$ holds, $\nabla\Lc_\text{SC}(\Hb)^\top (\Hb-\widehat\Hb) < 0$, which contradicts the necessary condition of local optimality \eqref{eq:1storder}.
}
{The rest of our results in this section, rely on \lem\ref{lem:stationary_NC}. Thus, in what follows, we assume
\begin{align*}
    \tau>{2}/{\log\big(\frac{n-1}{n_\text{{max}}-1}\big)}.
\end{align*}
For balanced data, $n_\text{max}=n/k$ and the condition approximately becomes $\tau>2/\log(k)$ for $n>>k$, which is not restrictive in practical settings.
For brevity, we do not re-state this condition in the next sections.
We finally note that the condition on $\tau$ is a sufficient requirement but not a necessary one. We conjecture that this condition can be relaxed. 
}

%% file: sections/SCL_ufm.tex
{
Here, we study the local and global optimality conditions for the UFM under the SC loss in \eqref{eq:SCL_UFM}. Since the problem is non-convex, it is not obvious whether \eqref{eq:SCL_UFM} has a unique solution or if it has a landscape free from spurious local solutions (which are not globally optimal). In fact, it is easy to see that the global solution is not unique, since rotating the embeddings by an orthonormal matrix $\Rb$, does not change the value of the loss function, i.e., $\Lc_\text{SC}(\Rb^\top\Hb)=\Lc_\text{SC}(\Hb)$. The next theorem characterizes the local and global solutions of \eqref{eq:SCL_UFM}.
}
{
\begin{theorem}[UFM landscape with SC loss]\label{thm:landscape_scl}
    Assume the feature dimension exceeds the number of classes, $d> k$. Then, 
    \\
    (i) all the local solutions of the UFM \eqref{eq:SCL_UFM} are global optima. \\
    (ii) if $\opt{\Hb}_1$ and $\opt{\Hb}_2$ are global minimizers, ${\opt{\Hb}_1}^\top\opt{\Hb}_1$=${\opt{\Hb}_2}^\top\opt{\Hb}_2$.
\end{theorem}
This theorem proves that although the global optimizer of the non-convex problem \eqref{eq:SCL_UFM} is not unique, it has a unique \emph{implicit geometry}. Furthermore, it ensures that the landscape of this problem has no non-global local solutions. We will next discuss the proof sketch.
}

\subsection{Proof Sketch}
{
The main idea in the proof is to relate the KKT points of the problem \eqref{eq:SCL_UFM} to the KKT points of its convex relaxation. We start by formalizing the convex relaxation.
}

\vspace{3pt}
{
\noindent\textbf{Convex relaxation.}~To define the convex relaxation, we note that the SC loss is in fact a convex function of the Gram matrix $\Gb:=\Hb^\top\Hb$.
\begin{align}
        \Lc_\text{SC}(\Hb)=\widehat\Lc(\Gb) &:= {\frac{1}{n}}\sum_{i\in[n]}\frac{-1}{n_{y_i}-1}\sum_{\substack{j \neq i\\ y_j =y_i}}\log\bigg(\frac{e^{\Gb_{i,j}}}{\sum\limits_{\ell \neq i}e^{\Gb_{i,\ell}}}\bigg).\nn
\end{align}
To make the convex loss function symmetric, we define $\Lc_{\text{cvx}}
    (\Gb):=\frac{1}{2}\Big(\widehat\Lc(\Gb) + \widehat\Lc(\Gb^\top)\Big)$ and we formulate the relaxed convex problem with the symmetrized objective function as follows.
\begin{align}\label{eq:SCL_UFM_cvx}
    \opt{\Gb} \in \arg\min_{\Gb} \Lc_{\text{cvx}}
    (\Gb),~\text{{subj. to}}~~\Gb\succcurlyeq0,~\Gb_{i,i} \leq \frac{1}{\tau},\, \forall i\in[n].
\end{align}
This convex program satisfies the Slater's conditions. Hence, strong duality holds and any KKT point of \eqref{eq:SCL_UFM_cvx} specifies a global solution. 
}

\vspace{3pt}
\noindent\textbf{Proof of \thm\ref{thm:landscape_scl}.}~{Suppose $\widehat{\Hb}$ is a local optima. Thus, it satisfies the first order and second order necessary condition for local optimality for constrained problems \cite[Theorems 12.1 and 12.5]{nocedal1999numerical}. 
In particular, consider the Lagrangian function $$\text{Lag}(\Hb,\lambdab) = \Lc_{\text{SC}}
    (\Hb) + \sum_{i\in[n]} \lambdab_i (\hb_i^\top\hb_i-{1}/{\tau}),$$ where $\lambdab\in\R^n$.
By the first-order conditions, there exists a dual variable $\widehat{\lambdab}\in\R^n$ such that $(\widehat{\Hb},\widehat{\lambdab})$ are the KKT points of \eqref{eq:SCL_UFM}. Specifically, $\nabla_\Hb \text{Lag}(\widehat\Hb,\widehat\lambdab)=0$, i.e.,
\begin{align*}
    2\widehat\Hb(\nabla_\Gb\Lc_\text{cvx}(\widehat\Hb^\top\widehat\Hb) +\diag(\widehat\lambdab))=0 
\end{align*}
By the second-order conditions, the Hessian of the Lagrangian 
    should have positive curvature in every direction $\Deltab_{d\times n} = [\db_1,...,\db_n]$ such that
\begin{align*}
    \db_i^\top\widehat\hb_i = 0,\quad \text{if } \norm{\widehat\hb_i}_2^2 = \frac{1}{\tau}, \, \widehat\lambda_i>0,\\
    \db_i^\top\widehat\hb_i \geq 0,\quad \text{if } \norm{\widehat\hb_i}_2^2 = \frac{1}{\tau}, \, \widehat\lambda_i=0,
\end{align*}

Suppose $\Bb:=\nabla_\Gb\Lc_\text{cvx}(\widehat\Hb^\top\widehat\Hb) +\diag(\widehat\lambdab)$ has a negative eigen-value, i.e., there exists $\ub\in\R^n$ such that $\ub^\top\Bb\ub<0$.
On the other hand, by \lem\ref{lem:stationary_NC}, the NC property holds at $\widehat\Hb$ and $\rank(\widehat\Hb)\leq k<d$. This implies that for some $\vb\in\R^d$, $\widehat\Hb^\top\vb = 0$. As a result, it is possible to show that The Hessian of the Lagrangian $\nabla^2_\Hb \text{Lag}(\widehat\Hb,\widehat\lambdab)$ has a negative curvature in the direction $\Deltab:=\vb\ub^\top$. 
This contradicts the second order condition of local optimality \cite[Theorems 12.5]{nocedal1999numerical} since $\Deltab^\top\widehat\Hb = 0$, i.e., $\db_i^\top\widehat\hb_i=0,\,\forall i\in[n]$.

Therefore, it is necessary for $(\widehat\Hb, \widehat{\lambdab})$ to satisfy,
$$
\nabla\Lc_\text{cvx}(\widehat{\Hb}^\top\widehat{\Hb})+\diag\big(\widehat\lambdab\big) \succcurlyeq 0.
$$
From here, it is straightforward to prove $(\widehat{\Gb}:=\widehat{\Hb}^\top\widehat{\Hb}, \widehat{\lambdab}, \Bb)$ is a KKT point of the convex problem \eqref{eq:SCL_UFM_cvx}. Equivalently, by strong duality, $\widehat{\Gb}$ is a global optimizer. On the other hand, since \eqref{eq:SCL_UFM_cvx} is a relaxation of the original UFM \eqref{eq:SCL_UFM}, for the optimal solution $\Hbopt$, we have,
\begin{equation}
    \Lc_\text{SC}(\widehat\Hb) = \Lc_\text{cvx}(\widehat\Gb) \leq \Lc_\text{cvx}({\Hbopt}^\top\Hbopt) = \Lc_\text{SC}(\Hbopt).
\end{equation}
Therefore, $\widehat\Hb$ is a global optimum of \eqref{eq:SCL_UFM}, and the proof is complete. }
%
This further shows the relaxation in \eqref{eq:SCL_UFM_cvx} is tight.
\begin{corollary}\label{cor:tight_scl}
    Suppose $d>k$. Then, the convex program \eqref{eq:SCL_UFM_cvx} is a tight relaxation of \eqref{eq:SCL_UFM}.
\end{corollary}

We can further show that the optimal solution of \eqref{eq:SCL_UFM_cvx} is unique. We use this useful result in the next section.
\begin{lemma}\label{lem:unique_main}
    The convex UFM problem \eqref{eq:SCL_UFM_cvx} has a unique minimizer.
\end{lemma}
For the proof of this lemma, we note two facts: 1) If $\Gb - \diag(\Gb) \propto \ones_n\ones_n^\top - \Id_n$, $\Gb$ is not a global optima, and 2) the objective $\Lc_\text{cvx}(\cdot)$ is strictly convex in any direction $\Vb$ such that
$
    \Vb - \diag(\Vb) \not\propto \ones_n\ones_n^\top - \Id_n.
    $

{
\subsection{Global Solutions}\label{sec:global_scl}}
\input{sections/SCL_global}

%% file: sections/SCL_global.tex
{
By \lem\ref{lem:stationary_NC}, we know that the global solution of the UFM \eqref{eq:SCL_UFM}, as a first order stationary point, has the NC property. That is, there exists $\opt\Mb_{d\times k}=[\opt\mub_1,...\opt\mub_k]$ such that $\opt\hb_i=\opt\mub_{y_i}$. Thus, to find the conditions of global optimality, we only need to find the $k$ optimal mean-embedding vectors $\opt\mub_c,\,c\in[k]$. 
}
To analyze the global solutions, we restrict our attention to a specific training data distribution defined as follows.

\begin{definition}[$(R,\rho)$-STEP imbalanced data] \label{def:step_data} With $R\geq 1$ as the imbalance ratio $\rho\in(0,1)$ as the minority fraction, the following holds: classes $c\in\{1,\ldots,\rhobar k\}$ are majority classes, all with the same sample size of $n_c=n_\text{maj}$ and classes $c\in\{\rhobar k + 1,\cdots, k\}$ are minority classes, all with the same sample size of $n_c=n_\text{minor}$, where $\nmaj=R\nmin$.

\end{definition}

{
%
Given the NC property at the optimal solution, in the $(R,\rho)$-STEP imbalanced setting, we can use the symmetry in the size of the classes to even further simplify the structure of the global solution.
The following lemma formalizes this argument.
\begin{lemma} \label{lem:scl_symm_global}  
    Assume an $(R,\rho)$-STEP imbalanced setting with $k$ classes. Let $\Hbopt$ be the global solution of \eqref{eq:SCL_UFM}. Then $\opt{\Gb}={\Hbopt}^\top\Hbopt$ is the unique optimizer of the convex program \eqref{eq:SCL_UFM_cvx}. Furthermore, $\opt{\Gb}$ has a block structure: there exists scalars $\alpha_\text{\emph{maj/minor}},\,\beta_\text{\emph{maj/minor}}, \,\theta \in[{-1}/{\tau},{1}/{\tau}]$ such that, 
    $\opt{\Gb} = \begin{bmatrix}
            \opt{\Gb}_{1,1} & \opt{\Gb}_{1,2}\\
            \opt{\Gb}_{2,1} & \opt{\Gb}_{2,2}
        \end{bmatrix},$ and 
    \begin{align*}
        \opt{\Gb}_{1,1} &= \bigg((\alpha_\emph{maj} - \beta_\emph{maj})\Id_{\rhobar k} + \beta_\emph{maj} \ones_{\rhobar k}\ones_{\rhobar k}^\top\bigg)\otimes \ones_{n_\emph{maj}} \ones_{n_\emph{maj}}^\top\\
        \opt{\Gb}_{1,2} &= \theta \ones_{\rhobar k n_\emph{maj}} \ones_{\rho k n_\emph{minor}}^\top \\
        \opt{\Gb}_{2,1} &= \theta \ones_{\rho k n_\emph{minor}} \ones_{\rhobar k n_\emph{maj}}^\top \\
        \opt{\Gb}_{2,2} &= \bigg((\alpha_\emph{minor} - \beta_\emph{minor})\Id_{\rho k} + \beta_\emph{minor} \ones_{\rho k}\ones_{\rho k}^\top\bigg)\otimes \ones_{n_\emph{minor}} \ones_{n_\emph{minor}}^\top
    \end{align*}
    %
\end{lemma}
This lemma follows from a simple symmetry argument that uses the fact that for two classes with equal sizes, their corresponding (mean-)embeddings are weighted similarly in the objective of problem \eqref{eq:SCL_UFM_cvx}:
\begin{lemma}\label{lem:symm_main}
    Consider problem \eqref{eq:SCL_UFM_cvx}. If $n_{c_1}=n_{c_2}$ for $c_1\neq c_2\in[k]$, then  $\forall i,j,\ell\in [n]$, such that $y_i=c_1,\,y_j=c_2$, and $y_\ell\neq c_1,c_2$, $\opt\Gb_{i,\ell} = \opt\Gb_{j,\ell}$.
\end{lemma}

Recall from \cor\ref{cor:tight_scl}, the relaxation \eqref{eq:SCL_UFM_cvx} is tight, and finding $\opt{\Gb}$ characterizes the optimal $\Hbopt$ up to a rotation. This lemma reduces the dimension of the original problem in the case of STEP imbalanced data. 
In the specific case of balanced data ($n_c=n/k$), the optimal solution reduces to ETF, as also previously shown \cite{graf2021dissecting}.
\begin{theorem}\label{thm:scl_ETF}
    Assume that the training set is balanced ($R=1$). The optimal solution $\Hbopt$ of the UFM \eqref{eq:SCL_UFM} follows an ETF. In other words, there exists vectors $\mub_1....\mub_k$ such that for $\Mb=[\mub_1,...\mub_k]$ we have $\Mb^\top\Mb \propto \Id_k-\frac{1}{k}\ones_k\ones_k^\top$ and $\opt{\hb}_i = \mub_{y_i}$.
\end{theorem}
}

{
This result follows directly from \lem\ref{lem:scl_symm_global}. Specifically, in the balanced case, $\alpha_\text{maj}=\alpha_\text{minor}$, $\beta_\text{maj}=\beta_\text{minor}$, and the optimal $\Gb^*$ satisfies,
\begin{align*}
    \Gb^* = \left(\alpha - \beta\right)\Id_n + \left((\beta-\theta)\Id_k + \theta \Jd_k\right) \otimes \Jd_{(\frac{n}{k})},
    \end{align*}
for some scalar $\alpha,\,\beta,\,\theta\in\R$. On the other hand, $\opt{\Gb}$ is feasible if it is positive semi-definite and its diagonal entries are less than $1/\tau$. Thus, we can show solving UFM \eqref{eq:SCL_UFM} reduces to optimizing a linear program as follows
\begin{align}
    (\alpha^*,\beta^*,\theta^*) = \arg\min_{\alpha,\beta,\theta}& \, \theta - \beta, \quad \\
    \text{sub. to} &\quad \frac{1}{\tau} - \alpha\geq 0,\nn\\
         &\alpha-\beta\geq0,\nn\\
         & \alpha + \beta (\frac{n}{k}-1) - \theta\frac{n}{k} \geq 0,\nn\\
         & \alpha + \beta(\frac{n}{k}-1) + \theta(n-\frac{n}{k}) \geq 0.\nn
\end{align}
From above, we can see $\alpha^*=\beta^*=1/\tau$ and $\theta^*=-1/(\tau(k-1))$, and ETF is the only solution. 
However, this is not the case if the training set is imbalanced:
\begin{propo}\label{lem:etf_counter_scl}{\cite[Lemma D.2]{SCL+}}
    If classes in the training set are not balanced, 
    the global solution of \eqref{eq:SCL_UFM} is not an ETF. 
\end{propo}
}
{
Overall, in balanced scenarios, as also previously shown \cite{graf2021dissecting,fang2021exploring}, SC learns the same embeddings as the CE loss. Similarly, the optimal solutions of the SC loss change with imbalances.
We leave further investigations for comparing the solutions in the imbalanced case for CE and SC loss for future work.
}

%% file: sections/discussion.tex
In this paper, complementing the claims and observations in \cite{graf2021dissecting,yaras2022neural}, we show that the SC loss has no spurious local optimal solutions under the UFM despite being non-convex. Although we leave the characterization of global solutions for general training distributions to future works, we exploit how the convex relaxation of the problem can be used for uncovering the properties of the unique global minimizer.
{We believe studying the SC loss is an intermediary step for expanding our understanding of embeddings geometry to the less-understood self-supervised setting.}
{While our focus is on discovering the local/global optimal embeddings geometries in the supervised contrastive learning, it is interesting to study the optimization dynamics analytically as a future direction. Previous works have empirically investigated the degree of convergence of the UFM and benchmark deep models to the optimal geometries for both the CE \cite{NC,zhu2021geometric,seli} and the SC  \cite{graf2021dissecting,SCL+} loss. Based on these reports, the training setup, such as the imbalance ratio or the loss hyper-parameters, play a significant role on the degree of convergence to the theoretical prediction by UFM. We hypothesize that the training setup can lead to more complex loss landscape, but gradient descent eventually leads to the global solutions asymptotically in the number of epochs. This conjecture remains to be further investigated and formalized.}

%% file: sections/lem1_pf.tex
\begin{lemma}[Gradient of SC] \label{lem:scl_grad}
    For a given matrix $\Hb_{d\times n}$, let $\Gb=\Hb^\top\Hb$ and define $\Pb_{n\times n}$ and $\Qb_{n\times n}$ as follows,
    \begin{align}
        \Pb_{i,j} &= \begin{cases}
        0, &i=j\\
        {\exp{(\Gb_{i,j}})}/{\sum_{\ell\neq i}\exp(\Gb_{i,\ell})}  &\text{\emph{o.w.}}\end{cases} \\
        \Qb_{i,j} &= \begin{cases}
        0, &i=j \,\text{\emph{  or  }}\,y_i\neq y_j\\
        1/(n_c-1)  &i\neq j,~y_i=y_j=c\end{cases} 
    \end{align}
    Then, the gradient of the SC loss $\Lc_\text{SC}(\cdot)$ and the convex SC loss $\Lc_\text{cvx}(\cdot)$ are as follows,
    \begin{align}
            \nabla_\Hb\Lc_\text{\emph{SC}}(\Hb) &= \Hb\big(\Pb+\Pb^\top)-2\Hb\Qb,\label{eq:scl_grad_h}\\
        \nabla_\Gb\Lc_\text{\emph{cvx}}(\Gb) &= \frac{1}{2}\big(\Pb+\Pb^\top) - \Qb.\label{eq:scl_grad_g}
    \end{align}
\end{lemma}
\begin{proof}
    Here, we prove \eqref{eq:scl_grad_h}. The proof of \eqref{eq:scl_grad_g} can be derived similarly. First, we simplify the expression of the loss.

\begin{align*}
    \Lc_\text{SC}(\Hb)  &= \sum_{i\in[n]}\frac{1}{n_{y_i}-1}\sum_{\substack{j\ne i\\y_j=y_i}}-\log\bigg(\frac{\exp(\hb_i^\top\hb_j)}{\sum_{\ell\neq i}\exp(\hb_i^\top\hb_\ell)}\bigg)\\
    %
    %
    &= \sum_{i\in[n]} \log\bigg({\sum_{\ell\neq i}\exp(\hb_i^\top\hb_\ell)}\bigg)-\sum_{i\in[n]}\frac{1}{n_{y_i}-1}\sum_{\substack{j\ne i\\y_j=y_i}}\hb_i^\top\hb_j.
\end{align*}
Now we can compute the gradient with respect to each column $\hb_i$ as follows,
\begin{align*}
      \frac{\partial \Lc_\text{SC}(\Hb)}{\partial \hb_i} &= \frac{{\sum_{\ell\neq i}\hb_\ell\exp(\hb_i^\top\hb_\ell)}}{{\sum_{\ell\neq i}\exp(\hb_i^\top\hb_\ell)}} + \sum_{j\neq i} \frac{{\hb_j\exp(\hb_j^\top\hb_i)}}{{\sum_{\ell\neq j}\exp(\hb_j^\top\hb_\ell)}} - \frac{2}{n_{y_i}-1}\sum_{\substack{j\neq i\\y_j=y_i}}\hb_j \\
    &=\sum_{j\neq i}\hb_j\Pb_{i,j} + \sum_{j\neq i} \hb_j\Pb_{j,i} - \frac{2}{n_{y_i}-1}\sum_{\substack{j\neq i\\y_j=y_i}}\hb_j\\
    &= \sum_{j\in[n]}\hb_j(\Pb_{i,j}+\Pb_{j,i}) - \sum_{\substack{j\neq i\\y_j=y_i}}\hb_j\Qb_{i,j}\\
    &=2\Hb(\pb_{i}+{\pb^i})-2\Hb\qb_i,
\end{align*}
    where  ${\pb^i}$ is the $i$-th row of $\Pb$, and $\pb_i,\,\qb_i$ are the $i$-th column of $\Pb$ and $\Qb$ respectively.
\end{proof}

\noindent\textbf{Proof of \lem\ref{lem:stationary_NC}.}~
Let $\Cc = \{\Hb\in\R^{d\times n}: \norm{\hb_i}^2\leq \frac{1}{\tau}\}$ be the convex feasible set, and suppose ${\Hb}\in\Cc$ is a feasible point that does not satisfy the NC property. Without loss of generality, we assume $y_1=y_2=1$ and ${\hb}_1\neq{\hb}_2$. Define $\widehat{\Hb} = [{\hb}_2,{\hb}_1,{\hb}_3,...,{\hb}_n]$ by exchanging the first and second column of ${\Hb}$. By the assumption $\widehat{\Hb}\neq{\Hb}$ and $\widehat{\Hb}-{\Hb}=[{\hb}_2-{\hb}_1,-({\hb}_2-{\hb}_1),0,...,0]$. Clearly, $\widehat{\Hb}$ remains feasible in the problem. We will show that 
    $f(\Hb):=\langle\nabla_{\Hb} \Lc_\text{SC}({\Hb}),\, \widehat{\Hb}-{\Hb}\rangle < 0.$
Recall the gradient of the SC loss from \lem\ref{lem:scl_grad}, and define $\Gb := \Hb^\top\Hb$.
\begin{align*}
    f(\Hb) &= \langle({\hb}_2-{\hb}_1),\frac{\partial \Lc_\text{SC}({\Hb})}{\partial {\hb}_1} - \frac{\partial \Lc_\text{SC}({\Hb})}{\partial {\hb}_2}\rangle\\
    &= ( {\hb}_2- {\hb}_1)^\top \Big(\sum_{j\in[n]} {\hb}_j(\Pb_{1,j}+\Pb_{j,1} - \Pb_{2,j} - \Pb_{j,2}) + \frac{2}{n_1-1}( {\hb}_1- {\hb}_2)\Big)\\
    &= \sum_{j\in[n]}( {\Gb}_{2,j}-  {\Gb}_{1,j})(\Pb_{1,j}+\Pb_{j,1} - \Pb_{2,j} - \Pb_{j,2}) - \frac{2}{n_1-1}\norm{ {\hb}_1- {\hb}_2}_2^2,\\
    &= \sum_{j\neq1,2}( {\Gb}_{2,j}-  {\Gb}_{1,j})(\Pb_{1,j}+\Pb_{j,1} - \Pb_{2,j} - \Pb_{j,2}) + (\Pb_{1,2}+\Pb_{2,1} - \frac{2}{n_1-1})\norm{ {\hb}_1- {\hb}_2}_2^2,
\end{align*}
Define,
\begin{align*}
    \alpha_1 &= \sum_{j\neq 1,2}( {\Gb}_{2,j}-  {\Gb}_{1,j})(\Pb_{1,j} - \Pb_{2,j})\\
    \alpha_2 &= \sum_{j\neq 1,2}( {\Gb}_{2,j}-  {\Gb}_{1,j})(\Pb_{j,1} - \Pb_{j,2})= \sum_{j\neq 1,2}( {\Gb}_{j,2}-  {\Gb}_{j,1})(\Pb_{j,1} - \Pb_{j,2})\\
    \alpha_3 &= (\Pb_{1,2}+\Pb_{2,1} - \frac{2}{n_1-1})\norm{ {\hb}_1- {\hb}_2}_2^2.
\end{align*}
Then, $f(\Hb_1) = \alpha_1 + \alpha_2 + \alpha_3$.
We will show if $\tau$ is large enough, $\alpha_1,\alpha_2\leq 0$ and $\alpha_3<0$.
We start by $\alpha_1$:
\begin{align*}
    \alpha_1 &= \sum_{j\neq 1,2}( {\Gb}_{2,j}-  {\Gb}_{1,j})(\Pb_{1,j} - \Pb_{2,j})\\
    &= \sum_{j\neq 1,2} ( {\Gb}_{2,j}-  {\Gb}_{1,j})\Bigg(\frac{\exp( {\Gb}_{1,j})}{\sum_{\ell\neq 1}\exp( {\Gb}_{1,\ell})} - \frac{\exp( {\Gb}_{2,j})}{\sum_{\ell\neq 2}\exp( {\Gb}_{2,\ell})}\Bigg)\\
    &= \frac{\sum_{j\neq 1,2} ({\Gb}_{2,j}-  {\Gb}_{1,j}) \bigg(\exp( {\Gb}_{1,j}) {\sum_{\ell\neq 2}\exp( {\Gb}_{2,\ell})} - \exp( {\Gb}_{2,j}){\sum_{\ell\neq 1}\exp( {\Gb}_{1,\ell})}\bigg)}{\bigg({\sum_{\ell\neq 1}\exp( {\Gb}_{1,\ell})}\bigg)\bigg({\sum_{\ell\neq 2}\exp( {\Gb}_{2,\ell})}\bigg)} =: \frac{\beta_1}{\beta_2}
\end{align*}
Since the exponential function is monotonic, we have
\begin{align*}
    \beta_1 &= \exp(\Gb_{1,2})\sum_{j\neq1,2}(\Gb_{2,j}-\Gb_{1,j})\Big(\exp(\Gb_{1,j})-\exp(\Gb_{2,i})\Big)\\
    &\quad +\sum_{\substack{j\neq1,2\\\ell\neq1,2}}((\Gb_{2,j}+\Gb_{1,\ell})-(\Gb_{2,\ell}+\Gb_{1,j}))\Big(\exp(\Gb_{2,\ell}+\Gb_{1,j})-\exp(\Gb_{2,j}+\Gb_{1,\ell})\Big)
    \leq 0.
\end{align*}
Thus $\alpha_1\leq0$. For $\alpha_2$, similarly with the monotonicity of the exponential function it is straightforward to show $\alpha_2\leq0$.
%
It remains to show $\alpha_3<0$.
Since $\norm{\hb_i}_2^2\leq1/\tau$, for all $i,j\in[n]$, $-1/\tau\leq\Gb_{i,j}\leq1/\tau$. 
Thus, for $\Pb_{1,2}$ we have
\begin{align*}
    \Pb_{1,2} \leq \frac{\exp(1/\tau)}{(n-1)\exp(-1/\tau)}.
\end{align*}
Now, suppose the condition on $\tau$ holds. Then, 
\begin{align*}
    \tau > \frac{2}{\log(\frac{n-1}{n_\text{max}-1})}\geq \frac{2}{\log(\frac{n-1}{n_1-1})} \quad \implies \quad \frac{1}{n_1-1} > \frac{\exp(1/\tau)}{(n-1)\exp(-1/\tau)}\geq\Pb_{1,2}
\end{align*}
We can show the above inequality similarly holds for $\Pb_{2,1}$. Thus,
\begin{align*}
    \Pb_{1,2}+\Pb_{2,1} - \frac{2}{n_1-1} <0 \quad \implies \quad \alpha_3<0,
\end{align*}
and $f(\Hb)<0$. This contradicts the first-order necessary condition for local optimality \eqref{eq:1storder}.
\qed

%% file: sections/lem2_pf.tex
\begin{lemma}[Strict-convexity of SC loss]\label{lem:strict_cvx}
For the convex SC loss $\Lc_\text{cvx}(\cdot)$, we have,
\begin{align}\label{eq:strict_cvx}
    \Lc_\text{\emph{cvx}}(\Gb+\Vb) > \Lc_\text{\emph{cvx}}(\Gb) + \tr\big(\Vb^\top \nabla \Lc_\text{\emph{cvx}}(\Gb)\big), \quad \text{if} \quad \Vb-\text{\emph{diag}}(\Vb)\not\propto\ones_n\ones_n^\top-\Id_n,
\end{align}
where $\text{\emph{diag}}(\Vb)\in\R^{n\times n}$ is a diagonal matrix whose diagonal entries are the same as $\Vb$.
    In other words, the loss is strictly-convex in any direction satisfying the condition in \eqref{eq:strict_cvx}.
\end{lemma}
\begin{proof}
For the vector $\gb\in\R^n$, define $f_{ij}(\gb) := \log\big(\sum_{\ell\neq i}e^{-\gb_j+\gb_\ell}\big)$ and 
let the rows and columns of $\Gb$ be defined as,
$$\Gb=\begin{bmatrix}
    \gb_1,&\cdots,&\gb_n
\end{bmatrix}=\begin{bmatrix}
    {\gb^1},&
    \cdots,&
    {\gb^n}
\end{bmatrix}^\top.$$
Then, we can rewrite the loss function as follows,
\begin{align*}
    \Lc_\text{{cvx}}(\Gb) = \frac{1}{2}\sum_{i\in[n]}\frac{1}{n_{y_i}-1}\sum_{\substack{j\neq i\\y_j=y_i}} (f_{ij}(\gb_i) + f_{ij}({\gb^i})).
\end{align*}
Define $\ab_i := \ones_n - \eb_i$, and $\Ab_i := (\Id_n - \frac{1}{n-1}\ab_i\ab_i^\top)$. We first show that for $j\in[n]$, $f(\gb):=f_{ij}(\gb)$ is strictly convex in any direction $\vb$ such that $\Ab_i\vb\neq0$. That is, for such $\vb$, $f(\gb+\vb) > f(\gb) + \vb^\top \nabla f(\gb)$. 

    From Taylor expansion, 
    $
        f(\gb+\vb) = f(\gb) + \vb^\top \nabla f(\gb) + \frac12\vb^\top \nabla^2 f(\gb+t\vb)\vb,
    $ for some $t\in(0,1)$.
    We will show that $\vb^\top \nabla^2 f(\gb)\vb>0$ for all $\gb$, which together with the Taylor expansion proves the statement. It is easy to verify that 
    $$
    \nabla^2 f(\gb) = \text{diag}(\ub) - \ub\ub^\top, \quad \ub_r = \begin{cases} 
    \frac{e^{-\gb_j+\gb_r}}{1+\sum_{\ell\neq i,j} e^{-\gb_j+\gb_\ell}}, & r \neq i\\
    0, & r=i
    \end{cases}.
    $$
    Thus, we equivalently need to show,
    $
        \vb^\top\text{diag}(\ub)\vb > (\ub^\top \vb)^2
    $. Since $\ub\geq0$, and $\ub^\top\ones_n = 1$,
    \begin{align*}
        \vb^\top\text{diag}(\ub)\vb &= \big(\sum_{r\in[n]} \vb_r^2 \ub_r \big)\big(\sum_{r\in[n]} \ub_r \big)\\
        &\geq \big( \sum_{i\in[n]} \vb_r\ub_r\big)^2 = (\ub^\top \vb)^2,
    \end{align*}
    where we used Cauchy-Schawrz inequality. The equality holds only if there exists a constant $\alpha$ such that for all $r\neq i$, $\vb_r = \alpha$ which contradicts the assumption $\Ab_i\vb\neq 0$. 

    Now consider direction $\Vb$ satisfying condition \eqref{eq:strict_cvx}. By convexity of $f_{ij}(\cdot)$, for all $i\in[n]$ and $j\in\{\ell|\,y_\ell=y_i,\,\ell\neq i\}$, we have,
\begin{align*}
    f_{ij} (\gb_{i} + \vb_{i}) \geq f_{ij} (\gb_{i}) + \vb_{i} ^ \top \nabla f_{ij} (\gb_{i}),\quad
    f_{ij} ({\gb^{i}} + {\vb^{i}}) \geq f_{ij} ({\gb^{i}}) + {\vb^{i}}^\top  \nabla f_{ij} (\gb^{i}).
\end{align*}
However, since $\Vb-\text{{diag}}(\Vb)\not\propto\ones_n\ones_n^\top-\Id_n$, there exists ${i'}\in[n]$ such that either  $\Ab_{i'}\vb_{i'}\neq 0$ or $\Ab_{i'}\vb^{i'}\neq 0$ (or both). Without loss of generality, suppose $\Ab_{i'}\vb_{i'}\neq 0$. Thus, by {strict-convexity}, for all $j\in\{\ell|\,y_\ell=y_i,\,\ell\neq i\}$, for $\gb_{i'}$ we have the following strict inequality.
    \begin{align*}
        f_{i'j} (\gb_{i'} + \vb_{i'}) > f_{i'j} (\gb_{i'}) + \vb_{i'} ^ \top \nabla f_{i'j} (\gb_{i'}).
    \end{align*}
    Summing the inequalities over all $i\in[n]$ and $j\in\{\ell|\,y_\ell=y_i,\,\ell\neq i\}$, we get \eqref{eq:strict_cvx}, and the proof is complete.
\end{proof}

\noindent\textbf{Proof of \lem\ref{lem:unique_main}}
    We first show that at the optimal solution, we meet the condition \eqref{eq:strict_cvx}. Suppose $\Gb^*-\diag(\Gb^*)=\alpha(\ones_n\ones_n^\top-\Id_n),\,\,\alpha\in\R$, and consider 
    \begin{align*}
        \Gb_1 = \frac{1}{\tau}\begin{bmatrix}
            \ones_{n_1}\ones_{n_1}^\top & 0\\
            0 & \ones_{(n-n_1)}\ones_{(n-n_1)}^\top
        \end{bmatrix}.
    \end{align*}
    $\Gb_1$ is positive semi-definite and feasible, so $\Lc(\Gb^*)\leq\Lc(\Gb_1)$. Then, we have,
    \begin{align*}
        \Lc(\Gb^*)\leq\Lc(\Gb_1) &= n_1\log\big((n_1-1) + (n-n_1)e^{-\frac{1}{\tau}}\big) + \sum_{c\neq1}n_c\log\big(n-n_1-1+n_1e^{-\frac{1}{\tau}}\big)\\
        &< n_1\log\big(n-1\big) + \sum_{c\neq1}n_c\log\big(n-1\big)\\
        &= n\log\big(n-1\big)\\
        &= \Lc(\Gb^*).
    \end{align*}
    That is $\Lc(\Gb^*)<\Lc(\Gb^*)$ which is contradiction. 

    Now, suppose $\Gb^*_1\neq\Gb^*_2$ are both optimal in $\eqref{eq:SCL_UFM_cvx}$. Since, both $\Gb_1^*$ and $\Gb_2^*$ are feasible, $\frac{1}{2}(\Gb_1^*+\Gb_2^*)$ is also feasible. Then, by optimality of $\Gb_1^*,\Gb_2^*$ and Lemma \ref{lem:strict_cvx},
    \begin{align*}
        \Lc(\Gb_1^*)\leq\Lc\Big(\frac{\Gb_1^*+\Gb_2^*}{2}\Big) < \frac{1}{2}(\Lc(\Gb_1^*)+\Lc(\Gb_2^*)) = \Lc(\Gb_1^*),
    \end{align*}
    which is a contradiction, and the proof is complete.
    \qed

%% file: sections/lem3_pf.tex
Recall from \lem\ref{lem:stationary_NC} that $\Hbopt$ satisfies the NC property. Thus, $\opt{\Gb}$, given the samples are ordered, has a block-form and can be parameterized by a symmetric $k\times k$ matrix as follows: 
\begin{align}\label{eq:G_NC}
    \exists\,\Gammab\in \R^{k\times k},\quad \text{s.t.}\quad \opt{\Gb}_{i,j} = \Gammab_{y_i,y_j}.
\end{align}
Also, recall from \lem\ref{lem:unique_main} that the global solution is unique.



Given \eqref{eq:G_NC}, we can find $\opt{\Gb}$, the solution of \eqref{eq:SCL_UFM_cvx}, by the following equivalent problem.
\begin{align}\label{eq:gamma_cvx}
      \opt{\Gammab},\opt{\Gb}\in\arg\min_{\Gammab,\Gb} f(\Gammab) : = \frac{1}{n}\sum_{c\in[k]}n_c&\log\bigg((n_c-1) + \sum_{c'\neq c}n_{c'}\exp({-\Gammab_{c,c}+\Gammab_{c,c'}})\bigg), \quad \\
      \text{s.t.}\quad &\Gammab_{c,c}\leq \frac{1}{\tau}\,\quad \forall c\in[k],\quad\nn\\
      &\Gb_{i,j}=\Gammab_{y_i,y_j}\,\quad \forall i,j\in[n]\nn\\
      &\Gb \succcurlyeq 0.\nn
\end{align}
We need the following lemma to complete the proof.
\begin{lemma}[Restatement of \lem\ref{lem:symm_main}]\label{lem:Gamma_symm}
    Consider problem \eqref{eq:gamma_cvx}. If $n_{c_1}=n_{c_2}$, we have
    \begin{align}\label{eq:gamma_cond}
        \opt{\Gammab}_{c_1,c_1} = \opt{\Gammab}_{c_2,c_2}, \qquad \opt\Gammab_{c,c_1} = \opt{\Gammab}_{c,c_2},\quad \forall c\neq c_1,c_2
    \end{align}
\end{lemma}
\begin{proof}
    Without loss of generality, let $c_1=1$ and $c_2=2$, and define the permutation matrix $\Rb:=\Rb_{c_1\,c_2}$ as follows:
    \begin{definition}[Permutation matrix] 
    For $i,j\in[k]$, the permutation matrix $\Rb_{ij}\in\R^{k\times k}$ is the matrix that swaps rows $i$ and $j$ of a matrix when multiplied by left. Equivalently, $\Rb_{ij}$ is the identity matrix $\Id_n$ with the rows $i$ and $j$ swapped.  
\end{definition}
Note that $\Rb^\top=\Rb$. For any $\Gammab\in\R^{k\times k}$, $\Rb\Gammab$ swaps row $c_1$ and $c_2$ and $\Gammab\Rb$ swaps columns $c_1$ and $c_2$.

Recall that $\opt{\Gammab}$ and $\opt{\Gb}$ are symmetric. We prove the lemma by contradiction. Specifically, if the lemmas condition does not hold, we show that by swapping rows $c_1$ and $c_2$ of $\Gammab$, followed by swapping the columns of the same index, the (optimal) loss value remains the same, while the optimizer changes. 

Suppose \eqref{eq:gamma_cond} does not hold. Define $\widetilde\Gammab=\Rb\opt{\Gammab}\Rb$. Since $n_1=n_2$, we have $f(\Gammab^*)=f(\widetilde{\Gammab})$, while by assumption, $\widetilde\Gammab\neq\opt{\Gammab}$. Now, define $\widetilde{\Rb}$ as follows,
\begin{align*}
    \widetilde{\Rb}:= \begin{bmatrix}
        0 & \Id_{n_1} & 0 & 0 &\cdots & 0\\
        \Id_{n_2} & 0 & 0 & 0 &\cdots & 0\\
        0 & 0 & \Id_{n_3} & 0 &\cdots & 0\\
        0 & 0 & 0 & \Id_{n_4} &\cdots & 0\\
        \vdots & \vdots & \vdots  & \vdots& \ddots & \vdots\\
        0 & 0 & 0 & 0 &\cdots & \Id_{n_k}
    \end{bmatrix}_{n\times n}.
\end{align*}

Then, it is straightforward to check $\Gbtilde=\widetilde{\Rb}\Gb^*\widetilde{\Rb}$ and $\widetilde\Gammab$ are feasible in the problem \eqref{eq:gamma_cvx}. Since $f(\Gammab^*)=f(\widetilde{\Gammab})$, $(\widetilde{\Gammab},\widetilde{\Gb})$ should also be optimal. This contradicts the uniqueness of the solution by Lemma \ref{lem:unique_main}, and the proof is complete.

\end{proof}

Given \lem\ref{lem:Gamma_symm}, we can show that in the STEP-imbalanced or balanced case (as a special case) setup, the global solution $\opt{\Gb}$ has a simpler form. Recall that in a $(R,\rho)$-STEP setting,
\begin{align*}
    n_c=n_{c'}\quad \text{if}\quad \begin{cases}
        c,c'\in\{1,...,\rhobar k\}\quad \text{(majority)}\\
        c,c'\in\{\rhobar k + 1, ..., k\} \quad \text{(minority)}
    \end{cases}
\end{align*}
Applying \eqref{eq:gamma_cond} for all pairs of classes with equal number of samples, we can show that there exists five scalars $\alpha_\text{maj},\alpha_\text{minor}$, $\beta_\text{maj},\beta_\text{minor}$, $\theta$ such that,
\begin{align*}
    \opt{\Gammab} = \begin{bmatrix}
        \Ab_{\text{maj}} && \theta\ones_{\rhobar k}\ones^\top_{\rho k}\\
        \theta\ones_{\rho k}\ones^\top_{\rhobar k} && \Ab_{\text{minor}}
    \end{bmatrix}, \qquad \begin{cases}
        \Ab_{\text{maj}} = (\alpha_\text{maj} - \beta_\text{maj})\Id_{\rhobar k} + \beta_\text{maj}\ones_{\rhobar k}\ones_{\rhobar k}^\top \\
    \Ab_{\text{minor}} = (\alpha_\text{minor} - \beta_\text{minor})\Id_{\rhobar k} + \beta_\text{minor}\ones_{\rhobar k}\ones_{\rhobar k}^\top
    \end{cases}.
\end{align*}
This proves \lem\ref{lem:scl_symm_global}. In the special case of balanced data, we can further show that $\alpha_\text{maj}=\alpha_\text{minor}=:\alpha$ and  $\beta_\text{maj}=\beta_\text{minor}=:\beta$. This reduces to,
\begin{align}\label{eq:balanced_G_app}
    \Gb^* = \left(\alpha - \beta\right)\Id_n + \left((\beta-\theta)\Id_k + \theta \Jd_k\right) \otimes \Jd_{(\frac{n}{k})},
\end{align}
where $\Jd_m=\ones_m\ones_m^\top$.
\qed